\documentclass{article}

\usepackage[margin=1in]{geometry}
\usepackage[numbers,compress]{natbib}
\usepackage[utf8]{inputenc}
\usepackage[T1]{fontenc}
\usepackage{hyperref}
\hypersetup{
  hidelinks,
  pdftitle={Image AID via continuous-time reinforcement learning},
  pdfauthor={Yilie Huang and Xun Yu Zhou}
}
\usepackage{url}
\usepackage{graphicx}
\usepackage{booktabs}
\usepackage{amsfonts}
\usepackage{amsmath,amssymb,amsthm,mathtools}
\usepackage{microtype}
\usepackage{xcolor}
\usepackage{bm}
\usepackage{algorithm}
\usepackage{algpseudocode}
\usepackage{float}

\title{Image AID via continuous-time reinforcement learning}

\author{Yilie Huang \thanks{Department of Applied Mathematics, The Hong Kong Polytechnic University, Hong Kong. Email: \texttt{yilie.huang@polyu.edu.hk}.}  ~ ~ ~ Xun Yu Zhou\thanks{Department of Industrial Engineering and Operations Research \& Data Science Institute, Columbia University, New York, NY 10027, USA. Email: \texttt{xz2574@columbia.edu}.}}
\date{}

\newcommand{\R}{\mathbb{R}}
\newcommand{\E}{\mathbb{E}}
\newcommand{\Id}{I}

\newcommand{\norm}[1]{\left\lVert #1 \right\rVert}

\newtheorem{theorem}{Theorem}[section]
\newtheorem{lemma}[theorem]{Lemma}

\begin{document}

\maketitle

\begin{abstract}
We study image inpainting with generative diffusion models. Existing methods typically either train dedicated task-specific models, or adapt a pretrained diffusion model separately for each masked image at deployment. We introduce a middle-ground model, termed Amortized Inpainting with Diffusion (AID), which keeps a pretrained diffusion backbone fixed, trains a small reusable guidance module offline, and then reuses it across masked images without per-instance optimization. We formulate it as a deterministic guidance problem with a supervised terminal objective. To make this problem learnable in high dimensions, we derive an auxiliary Gaussian formulation and prove that solving this randomized problem recovers the optimal deterministic guidance field. This bridge yields a principled continuous-time actor--critic algorithm for learning the guidance module in a fully data-driven manner. Empirically, on AFHQv2 and FFHQ under the pixel EDM pipeline and on ImageNet under the latent EDM2 pipeline, AID consistently improves the quality--speed trade-off over strong fixed-backbone and amortized inpainting baselines across multiple mask types, while adding less than one percent trainable overhead.
\end{abstract}

\section{Introduction}

Image inpainting is a conditional generation problem: given a partially observed image (e.g. a damaged photo), the goal is to complete the missing region in a way that is both visually plausible and consistent with the visible pixels. Recent diffusion models \cite{ho2020denoising,song2020denoising,nichol2021improved,rombach2022high} have made this task substantially more effective, especially when strong pretrained generative models are available. This makes inpainting a particularly important setting for studying how a powerful pretrained generative model can be adapted to conditional tasks efficiently and accurately.

Despite this progress, the current research remains unsatisfactory. Existing approaches are largely split between two directions. One is to train a dedicated inpainting model from scratch or a substantial task-specific adaptation module, as in methods such as LaMa \cite{suvorov2022resolution}, MAT \cite{li2022mat}, Palette \cite{saharia2022palette}, BrushNet \cite{ju2024brushnet}, and DAVI \cite{lee2024diffusion}. This direction amortizes inference, but does so by adding large learned components and requiring substantial training time. The other route is to keep a pretrained diffusion model fixed and adapt it separately for each masked image at test time, as in methods such as RePaint \cite{lugmayr2022repaint}, DPS \cite{chung2022diffusion}, and MCG \cite{chung2022improving}. This preserves the pretrained model, but each masked image is still treated as a new inference problem, leading to repeated case-by-case optimization or iterative refinement during deployment. In other words, one direction amortizes well but pays in learned-model complexity, while the other reuses the pretrained model but pays repeatedly at test time. There is a gap in the middle ground.


Our answer to this gap is an \emph{Amortized Inpainting with Diffusion} (AID)  model, which keeps the pretrained diffusion model fixed and trains only a small reusable guidance module offline. Once trained, the same module is reused across masked images at test time; so inference no longer requires solving a new optimization problem for every task. AID is built on a controlled-guidance view of inpainting: the guidance deployed for actual tasks is deterministic, while learning uses randomized Gaussian policies. A central contribution of the paper is to prove that this randomized learning problem is optimizer-preserving, making the resulting actor--critic algorithm theoretically grounded rather than heuristic. 

The main contributions of this paper are as follows:
\begin{itemize}
\item
{\em Methodology}:
We propose AID, a controlled-guidance framework for image inpainting that keeps a pretrained diffusion model fixed and trains offline a lightweight guidance module once. The learned module is then reused across masked images at test time, combining the reuse advantage of pretrained diffusion models with the efficiency of amortized inference.

\item
{\em Theory}:
We develop a rigorous theoretical foundation for AID by establishing a policy-equivalence bridge between the deterministic guidance problem and an auxiliary Gaussian actor--critic formulation. Specifically, we prove both directions of the correspondence: the deterministic optimal guidance is sufficient to construct an optimal auxiliary Gaussian policy, and any optimal auxiliary Gaussian policy must recover this deterministic guidance through its mean. This optimizer-preserving bridge justifies learning with randomized Gaussian policies while deploying the deterministic policy mean.

\item
{\em Experiments}:
We evaluate AID on AFHQv2, FFHQ, and ImageNet under multiple mask types and metrics, with well under \(1\%\) trainable overhead relative to the pretrained diffusion model.
Across fixed-backbone baselines, the default setting consistently delivers the strongest performance and surpasses RePaint using only about \(10\%\) of its sampling budget in both pixel and latent spaces. The same trained guidance module transfers without retraining to unseen masks and to a lower-latency sampler, which further reduces sampling time while still outperforming all non-RePaint baselines and remaining competitive with RePaint.
\end{itemize}

To our best knowledge, AID is the first amortized controlled-guidance framework for image inpainting that keeps the pretrained diffusion backbone frozen, trains only a small reusable guidance module, and provides an exact optimizer-preserving bridge from stochastic actor--critic learning to the deterministic guidance field used at deployment.

{\bf Relevant literature}:
In addition to the existing inpainting research discussed earlier on two major directions,
we mention a related control-based work that largely follows the second direction. \cite{li2024solving} formulates inpainting as a test-time control problem, but it still requires instance-wise iterative optimization and incurs substantial sampling cost. The closest middle-ground comparator is LatentPaint \cite{corneanu2024latentpaint}, a practical propagation-based method for reusable inpainting adaptation. While it shares our goal of amortized adaptation, it is tied to latent-space updates. AID instead learns principled reusable guidance and applies to both pixel-space EDM and latent-space EDM2 backbones.

Continuous-time reinforcement learning (CT-RL) provides the methodological foundation for our learning formulation, while the optimizer-preserving Gaussian bridge developed below is specific to AID. Since modern diffusion models for generative AI are naturally described through continuous-time SDEs and ODEs \cite{song2020score}, CT-RL is a particularly natural framework for our problem. This line of work was initiated by \citet{wang2020reinforcement} and has since developed into a model-free theory based on martingale methods \citep{jia2021policy, jia2021policypg, tang2024regret}, with further results on performance guarantees \citep{huang2024mean, huang2025sublinear}. More recently, CT-RL has also begun to appear in diffusion model training, fine-tuning, and timestep selection \citep{gao2025reward, ZC25, huang2026art}. Recent work on diffusion fine-tuning also adapts pretrained diffusion models toward desired objectives or conditioning information \citep{uehara2024finetuning,denker2024deft}. AID instead focuses on supervised image inpainting and learns a small reusable guidance module while keeping the diffusion backbone fixed. A related discrete-time RL-based inpainting work is PrefPaint \cite{liu2024prefpaint}, which fine-tunes the model using large-scale human preference annotations and thus relies on human-driven, subjective supervision rather than our purely data-driven setting. A shorter version of this paper has been accepted for presentation at NeurIPS 2026 \citep{huang2026neurips}.

{\bf Organization}:
Section~\ref{sec:formulation} formulates amortized inpainting as a controlled diffusion problem. Section~\ref{sec:bridge} proves the optimizer-preserving Gaussian bridge, and Section~\ref{sec:algorithm} turns it into the AID actor--critic solver. Section~\ref{sec:experiments} reports experiments; proofs and additional results are in the appendix.

\section{Control Formulation for Amortized Inpainting}
\label{sec:formulation}

\paragraph{Setup.}
We first introduce the problem of amortized inpainting over a task distribution. Let \(x^\dagger \in \R^d\) denote a clean reference image (from a training set), \(M \in \R^{d \times d}\) be a diagonal binary mask matrix selecting the visible region, and \(\Id-M\) select the missing region. We assume that \((x^\dagger,M)\) is sampled from a task distribution \(\rho_{\mathrm{task}}\), and define the masked observation by \(y := Mx^\dagger\). The observable input is then \(\xi := (M,y)\), while the full supervised tuple is \(\zeta := (\xi,x^\dagger) = (M,y,x^\dagger)\). Thus, \(\xi\) is the information available for a specific inpainting task, whereas \(\zeta\) is available only during offline training. Throughout, \(\norm{\cdot}\) denotes the Euclidean norm on \(\R^d\). The key information structure is that deployment uses only \(\xi\), while offline training may use \(\zeta\).

Let \(t\in[0,T]\) denote diffusion time, where the forward process follows
\[
\mathrm{d}\bar{x}(t)
=
-f(t)\bar{x}(t)\,\mathrm{d}t
+
g(t)\,\mathrm{d}w(t),
\qquad
\bar{x}(0)\sim p_0,
\]
where \(p_0\) is the data distribution of clean images. Let \(p_t\) denote the law of \(\bar x(t)\) and let \(S(t,x)=\nabla_x \log p_t(x)\) be the score. In our setting, \(S\) is replaced by a pretrained and fixed approximation \(\hat S\). Then the associated probability-flow ordinary differential equation \cite{song2020score} is
\[
\dot x(t)=b(t,x),
\qquad
x(0)\sim p_T,
\]
where the reverse drift is defined by \(b(t,x):=f(T-t)x+\frac{1}{2}g^2(T-t)\hat S(T-t,x)\).

\paragraph{Guided reverse dynamics.}
The drift \(b(t,x)\) captures the frozen reverse dynamics induced by the pretrained diffusion model. For inpainting, however, this drift alone does not explicitly enforce consistency with the masked observation. Motivated by the general idea of diffusion guidance \cite{dhariwal2021diffusion,ho2022classifier,meng2021sdedit}, we therefore add a task-dependent control term for guidance and consider
\begin{equation}
\dot x(t)=b(t,x(t))+u(t,x(t);\xi),
\qquad
x(t_0)=x_0,
\label{eq:controlled}
\end{equation}
where \(u\) is the guidance field. The information structure in \eqref{eq:controlled} is that the deployed guidance law may depend on the mask and visible pixels, \(\xi=(M,y)\), but not on the hidden clean image \(x^\dagger\).

\paragraph{Terminal objective.}
For a supervised tuple \(\zeta=(M,y,x^\dagger)\), we define the terminal loss by
\[
\Psi(x;\zeta)
=
\alpha_{\mathrm{vis}}\psi_{\mathrm{vis}}(x;\zeta)
+
\alpha_{\mathrm{hole}}\psi_{\mathrm{hole}}(x;\zeta),
\]
where \(\alpha_{\mathrm{vis}},\alpha_{\mathrm{hole}}\ge 0\) balance visible-region fidelity and missing-region fidelity,
\[
\psi_{\mathrm{vis}}(x;\zeta):=\frac{1}{2}\norm{M(x-x^\dagger)}^2,
\qquad
\psi_{\mathrm{hole}}(x;\zeta):=\frac{1}{2}\norm{(\Id-M)(x-x^\dagger)}^2.
\]
The first term enforces agreement with the observed pixels, while the second supervises the missing region using the clean image \(x^\dagger\). Since the deployed guidance law depends only on the observable input \(\xi=(M,y)\), the corresponding terminal target in the control problem is the conditional expectation \(\bar\Psi(x;\xi):=\E[\Psi(x;\zeta)\mid \xi]\). During offline training, sampled supervised tuples \(\zeta\) provide Monte Carlo realizations of this conditional target.

\paragraph{Control objective.}
Let \(\mathcal U(\xi)\) denote the class of admissible feedback control fields \((t,x)\mapsto u(t,x;\xi)\in\R^d\). For \(u\in\mathcal U(\xi)\), define
\begin{equation}
J^{u}(t,x;\xi)
:=
\bar\Psi\bigl(x^u(T;t,x,\xi);\xi\bigr)
+
\int_t^T
\frac{\beta}{2}\,
\norm{u(r,x^u(r;t,x,\xi);\xi)}^2\,\mathrm{d}r,
\label{eq:cost}
\end{equation}
where \(\beta>0\) is a scalar control weight. The terminal term evaluates final inpainting quality averaged over the supervision with \(\xi\), while the running term penalizes the cumulative strength of the guidance field along the trajectory. The latter reflects the notion that the pretrained reverse dynamics already provide a strong image model, and the additional guidance should intervene only when necessary to enforce consistency with the inpainting task. Although training is over samples from the task distribution, the deployed controller observes the realized masked input \(\xi\) to deploy a control. Accordingly, the underlying control problem is formulated conditionally on \(\xi\):
\begin{equation}
V(t,x;\xi):=\inf_{u\in\mathcal U(\xi)} J^u(t,x;\xi).
\label{eq:value}
\end{equation}
Under standard regularity assumptions ensuring well-posedness, \(V\) is the unique classical solution of
\begin{equation}
V_t(t,x;\xi)
+
\inf_{u\in\R^d}
\left\{
\nabla_x V(t,x;\xi)^\top \bigl(b(t,x)+u\bigr)
+
\frac{\beta}{2}\norm{u}^{2}
\right\}
=0,
\label{eq:hjb}
\end{equation}
with terminal condition
\begin{equation}
V(T,x;\xi)=\bar\Psi(x;\xi).
\label{eq:terminal}
\end{equation}
This completes the control formulation of our amortized inpainting diffusion (AID) model. 

\section{A Theoretical Bridge from Deterministic Guidance to Continuous-Time RL}
\label{sec:bridge}

Section~\ref{sec:formulation} formulates AID as a deterministic control problem whose optimizer is the guidance used at deployment. The central challenge is that this problem is high dimensional and induced by a highly nonlinear pretrained diffusion model; so the solution is generally not tractable analytically nor feasible numerically. To make this control problem tractable, we introduce an auxiliary randomized control problem with a carefully chosen Gaussian policy class.

This auxiliary formulation is useful only if it preserves the original deterministic control target. In particular, randomization should not change the guidance field ultimately deployed for inpainting. This section establishes this relationship: we show that solving the auxiliary randomized problem recovers exactly the deterministic optimal guidance field of AID. Only after this equivalence is established can the randomized formulation serve as the basis for continuous-time RL and actor--critic updates in Section~\ref{sec:algorithm}.

\subsection{Auxiliary Gaussian Formulation}

Motivated by continuous-time relaxed-control formulations with Gaussian randomization \citep{huang2025data,huang2025continuous}, for each fixed observable input \(\xi\) and each \(\lambda>0\), let \(\Pi^{(\lambda)}(\xi)\) denote the class of Gaussian feedback policies of the form
\[
\pi^{(\lambda)}(\cdot\mid t,x;\xi)
=
\mathcal N\!\left(
\mu^{\pi^{(\lambda)}}(t,x;\xi),
\frac{2\lambda}{\beta d}\Id
\right).
\]
Thus, under \(\pi^{(\lambda)}\), the control is sampled from a Gaussian policy, while the policy mean is the candidate deterministic guidance law. Unlike entropy-regularized exploratory CT-RL, we do not add an entropy term to the objective; the Gaussian randomization is introduced as a learning device whose mean is designed to remain aligned with the deterministic guidance field. Following the relaxed-control formulation of \cite{wang2020reinforcement}, the corresponding relaxed dynamics under a policy \(\pi^{(\lambda)}\) are
\[
\dot x^{\pi^{(\lambda)}}(r)
=
b\bigl(r,x^{\pi^{(\lambda)}}(r)\bigr)
+
\int_{\R^d} a\,\pi^{(\lambda)}(a\mid r,x^{\pi^{(\lambda)} }(r);\xi)\,\mathrm{d}a,
\qquad
x^{\pi^{(\lambda)}}(t)=x,
\]
while the associated randomized objective is written in centered form
\[
J^{\pi^{(\lambda)}}(t,x;\xi)
:=
\bar\Psi\bigl(x^{\pi^{(\lambda)}}(T;t,x,\xi);\xi\bigr)
-
\lambda T
+
\int_t^T
\left(
\int_{\R^d}\frac{\beta}{2}\norm{a}^{2}\,
\pi^{(\lambda)}(a\mid r,x^{\pi^{(\lambda)}}(r);\xi)\,\mathrm{d}a
\right)\mathrm{d}r.
\]
The shift \(-\lambda T\) centers the auxiliary value; being policy-independent, it leaves all optimizers unchanged.
Next we define
\begin{equation}
V^{(\lambda)}(t,x;\xi)
:=
\inf_{\pi^{(\lambda)}\in\Pi^{(\lambda)}(\xi)}
J^{\pi^{(\lambda)}}(t,x;\xi).
\label{eq:aux_value}
\end{equation}
Under the same regularity assumptions as before, \(V^{(\lambda)}\) is the unique classical solution of
\begin{equation}
V_t^{(\lambda)}(t,x;\xi)
+
\inf_{\mu\in\R^d}
\left\{
\nabla_x V^{(\lambda)}(t,x;\xi)^\top \bigl(b(t,x)+\mu\bigr)
+
\frac{\beta}{2}\norm{\mu}^{2}
+
\lambda
\right\}
=0,
\label{eq:aux_hjb}
\end{equation}
with terminal condition
\begin{equation}
V^{(\lambda)}(T,x;\xi)=\bar\Psi(x;\xi)-\lambda T.
\label{eq:aux_terminal}
\end{equation}

\subsection{Correspondence between Auxiliary Gaussian Problem and AID}

The auxiliary Gaussian formulation is useful only if it gives rise to the (deterministic) inpainting guidance in the original problem. We first record a value-shift lemma, which aligns the deterministic and auxiliary HJB equations. We then establish the main policy correspondence in two directions: the deterministic AID optimizer is sufficient to construct an optimal auxiliary Gaussian policy, and conversely, any optimal auxiliary Gaussian policy must recover the deterministic optimizer through its mean. Detailed proofs are given in Sections~\ref{app:proof_value_shift}, \ref{app:proof_policy}, and \ref{app:proof_learning_equivalence}.

\begin{lemma}[Centered value shift for the auxiliary problem]
\label{lem:value_shift}
Suppose \eqref{eq:hjb}--\eqref{eq:terminal} and \eqref{eq:aux_hjb}--\eqref{eq:aux_terminal} admit unique classical solutions \(V\) and \(V^{(\lambda)}\), respectively. Then, for every fixed observable input \(\xi\),
\[
V^{(\lambda)}(t,x;\xi)=V(t,x;\xi)-\lambda t.
\]
\end{lemma}

Lemma~\ref{lem:value_shift} shows that the auxiliary value differs from the deterministic value only by a known normalization. This value relation is an important alignment step, but by itself it does not yet justify using randomized policies for learning. The next two results establish the stronger policy-level correspondence.

\begin{theorem}[Sufficiency: deterministic guidance induces an optimal Gaussian policy]
\label{thm:policy_sufficiency}
Under the assumptions of Lemma~\ref{lem:value_shift}, define
\[
u^*(t,x;\xi):=-\frac{1}{\beta}\nabla_x V(t,x;\xi).
\]
Then the Gaussian policy
\begin{equation}
\pi^{(\lambda),*}(\cdot\mid t,x;\xi)
=
\mathcal N\!\left(
u^*(t,x;\xi),
\frac{2\lambda}{\beta d}\Id
\right)
\label{eq:pi_star}
\end{equation}
is optimal for the auxiliary problem \eqref{eq:aux_value} within \(\Pi^{(\lambda)}(\xi)\).
Moreover, \(u^*\) solves the AID problem \eqref{eq:value}.
\end{theorem}

Theorem~\ref{thm:policy_sufficiency} shows that passing from the deterministic AID problem to the auxiliary Gaussian problem does not lose the desired control: the deterministic optimizer can be lifted to an optimal randomized policy by using it as the Gaussian mean. For learning, however, the reverse direction is equally important. If the auxiliary problem is solved by a learning algorithm, we must also know that an optimal auxiliary policy cannot have a different mean. This is the necessity direction below.

\begin{theorem}[Necessity: auxiliary optimality recovers deterministic guidance]
\label{thm:policy_necessity}
Under the assumptions of Lemma~\ref{lem:value_shift}, let
\(\pi^{(\lambda),*}\in\Pi^{(\lambda)}(\xi)\) be any optimal feedback policy for the auxiliary Gaussian problem \eqref{eq:aux_value} for every initial pair \((t,x)\). Then its mean field \(\mu^{\pi^{(\lambda),*}}(t,x;\xi)\) is the optimal deterministic guidance field for the original AID problem \eqref{eq:value}; that is,
\[
\mu^{\pi^{(\lambda),*}}(t,x;\xi)=u^*(t,x;\xi).
\]
\end{theorem}

Together, Lemma~\ref{lem:value_shift} and Theorems~\ref{thm:policy_sufficiency}--\ref{thm:policy_necessity} show that the auxiliary Gaussian formulation changes the learning mechanism but not the control target. Sufficiency lifts the deterministic AID optimizer to an optimal Gaussian policy, while necessity rules out auxiliary optima with different means. Hence the actor in Section~\ref{sec:algorithm} can be parameterized as the Gaussian mean and deployed deterministically.

\section{Amortized Actor--Critic Algorithm}
\label{sec:algorithm}

Lemma~\ref{lem:value_shift} and Theorems~\ref{thm:policy_sufficiency}--\ref{thm:policy_necessity} identify the learning target: solving the auxiliary Gaussian problem is equivalent, at the policy-mean level, to learning the deterministic optimal guidance field \(u^*\). This implication is essential, because policy-evaluation and policy-gradient methods in CT-RL rely on stochastic exploration and therefore do not apply directly to the original deterministic control problem. At the same time, our auxiliary Gaussian formulation is not the entropy-regularized exploratory objective studied in \cite{jia2021policy,jia2021policypg}; the Gaussian variance is fixed to preserve the deterministic optimizer rather than optimized through an entropy reward. We therefore adapt their martingale orthogonality and policy-gradient methodology to the AID-specific Bellman residual, leading to moment conditions and stochastic-approximation updates specialized to our auxiliary problem.

\paragraph{Parametrization.}
We parameterize the critic and actor by
\[
\hat V^\theta(t,x;\xi):=NN^\theta(t,x,\xi)-\lambda t,
\qquad
\hat\mu^\phi(t,x;\xi):=NN^\phi(t,x,\xi).
\]
The associated Gaussian policy is
\(
\hat\pi^{(\lambda),\phi}(\cdot\mid t,x;\xi)
:=
\mathcal N\!\left(
\hat\mu^\phi(t,x;\xi),
\frac{2\lambda}{\beta d}\Id
\right).
\)

\paragraph{Martingale moment conditions.}
Let \(A_t\sim\pi^{(\lambda)}(\cdot\mid t,X_t;\xi)\), and let \(X_t\) be the state process obtained by rolling out the guided reverse dynamics under this exploratory policy during training. Thus, \(\mathrm{d}\hat V^\theta(t,X_t;\xi)\) below denotes the critic increment along this same sampled-action trajectory. For the auxiliary Gaussian formulation, the quantity
\(
\mathrm{d}V^{(\lambda)}(t,X_t;\xi)+\frac{\beta}{2}\norm{A_t}^{2}\,\mathrm{d}t
\)
plays the role of the Bellman residual. Following the continuous-time actor--critic framework, we require this residual to be orthogonal to the critic and actor test functions. This yields the coupled moment conditions
\[
\E\!\left[
\int_0^T
\partial_\theta \hat V^\theta(t,X_t;\xi)
\left(
\mathrm{d}\hat V^\theta(t,X_t;\xi)
+\frac{\beta}{2}\norm{A_t}^{2}\,\mathrm{d}t
\right)
\right]
=0,
\]
\[
\E\!\left[
\int_0^T
\partial_\phi \log \hat\pi^{(\lambda),\phi}(A_t\mid t,X_t;\xi)
\left(
\mathrm{d}\hat V^\theta(t,X_t;\xi)
+\frac{\beta}{2}\norm{A_t}^{2}\,\mathrm{d}t
\right)
\right]
=0,
\]
where the expectation is taken over the sampled task \(\zeta\sim\rho_{\mathrm{task}}\), the initial state, and the Gaussian policy randomness. The critic is anchored by the auxiliary terminal condition, implemented below through the terminal target \(\hat V_{n,K}\). The discretized updates below are the corresponding stochastic-approximation discretizations of these two moment conditions.

\paragraph{Time discretization and updates.}
The problem is formulated in continuous time because both the diffusion dynamics and the results in Section~\ref{sec:bridge} are inherently continuous-time; time discretization is introduced only for numerical implementation. We implement the method on a time grid \(0=t_0<\cdots<t_K=T\) with \(\Delta t_k:=t_{k+1}-t_k\). We use Heun's method, hold the sampled action fixed across the two drift evaluations within each step, and use Euler for the final step; hence the number of function evaluations is \(2K-1\) (35 for \(K=18\), and 23 for \(K=12\)). At iteration \(n\), we sample \(\zeta_n=(M_n,y_n,x_n^\dagger)\sim\rho_{\mathrm{task}}\), set \(\xi_n=(M_n,y_n)\), and roll out one trajectory under \(\hat\pi^{(\lambda),\phi_n}\). Writing \(\hat V_{n,k}:=\hat V^{\theta_n}(t_k,X_{n,k};\xi_n)\), \(\hat\mu_{n,k}:=\hat\mu^{\phi_n}(t_k,X_{n,k};\xi_n)\), and \(\hat V_{n,K}:=\Psi(X_{n,K};\zeta_n)-\lambda T\), we define the one-step residual
\[
\delta_{n,k}:=\hat V_{n,k+1}-\hat V_{n,k}+\frac{\beta}{2}\norm{A_{n,k}}^{2}\Delta t_k,
\]
which is the discretized analogue of a temporal-difference residual. The updates
\begin{align}
\theta_{n+1}
\leftarrow\;&
\theta_n
+
a_n^{c}\sum_{k=0}^{K-1}
\partial_\theta \hat V^{\theta_n}(t_k,X_{n,k};\xi_n)\,
\delta_{n,k},
\label{eq:critic_update}
\\[0.3em]
\phi_{n+1}
\leftarrow\;&
\phi_n
-
a_n^{a}\sum_{k=0}^{K-1}
\partial_\phi \log \hat\pi^{(\lambda),\phi_n}(A_{n,k}\mid t_k,X_{n,k};\xi_n)\,
\delta_{n,k},
\label{eq:actor_update}
\end{align}
are therefore the continuous-time analogues of actor--critic updates in discrete-time RL \citep{sutton1988learning,konda1999actor}. Moreover, for the Gaussian policy,
\(
\partial_\phi \log \hat\pi^{(\lambda),\phi}(A\mid t,x;\xi)
=
\frac{\beta d}{2\lambda}\,
\left(\partial_\phi \hat\mu^\phi(t,x;\xi)\right)^\top
\left(A-\hat\mu^\phi(t,x;\xi)\right).
\)
This also shows why the auxiliary Gaussian formulation is necessary: the stochastic term \(A-\hat\mu^\phi\) is exactly what makes the policy-gradient update nondegenerate, whereas under a deterministic policy this score term vanishes and the policy-gradient identity becomes uninformative.

Algorithm~\ref{alg:aid} summarizes the offline training and deployment procedures of AID. The amortization lies in learning shared actor and critic parameters over tasks sampled from \(\rho_{\mathrm{task}}\), and then reusing the learned mean guidance law on unseen masked inputs without per-instance optimization.

\begin{algorithm}[t]
\caption{Offline Training and Deployment of AID}
\label{alg:aid}
\begin{algorithmic}[1]
\For{$n=1,\dots,N$}
    \State Sample \(\zeta_n=(M_n,y_n,x_n^\dagger)\sim \rho_{\mathrm{task}}\), set \(\xi_n=(M_n,y_n)\), and initialize \(X_{n,0}\sim p_T\)
    \For{$k=0,\dots,K-1$}
        \State Compute \(\hat\mu_{n,k}=\hat\mu^{\phi_n}(t_k,X_{n,k};\xi_n)\) and sample \(A_{n,k}\sim \mathcal N(\hat\mu_{n,k}, \frac{2\lambda}{\beta d}\Id)\)
        \State Advance one Heun step, holding \(A_{n,k}\) fixed within the step, to obtain \(X_{n,k+1}\)
    \EndFor
    \State Update \(\theta_n\) and \(\phi_n\) by \eqref{eq:critic_update} and \eqref{eq:actor_update}
\EndFor
\Statex
\State \textbf{Deployment:} for a new masked input \(\xi\), discard policy noise and use the deterministic mean guidance \(\hat\mu^\phi(t_k,X_k;\xi)\) inside the same guided reverse-time solver.
\end{algorithmic}
\end{algorithm}

\section{Empirical Results}
\label{sec:experiments}
\paragraph{Backbones and datasets.}
We evaluate AID within the official EDM and EDM2 pipelines \cite{karras2022elucidating,karras2024analyzing,karras2024guiding}, which are representative modern diffusion backbones designed for strong image quality at low sampling budgets (\(K=18\) steps). Specifically, we use AFHQv2 \cite{choi2020stargan} and FFHQ \cite{karras2019style} with the official EDM backbones, and ImageNet \cite{deng2009imagenet,russakovsky2015imagenet} with the official EDM2 backbone. These are the highest-resolution benchmark settings released in the corresponding official EDM/EDM2 pipelines used here, and together they let us test AID across two different pretrained diffusion families in both a pixel-space regime and a higher-resolution latent-diffusion regime.

\paragraph{Baselines and scope.}
Our comparison set is designed to test the central deployment question: how much can be gained by amortizing a small guidance module while keeping the pretrained diffusion backbone fixed? Full inpainting networks and large learned helper modules, such as LaMa, BrushNet, and DAVI, test a different deployment regime: they amortize inference by training a model-scale component, whereas our goal is to isolate the cost--quality trade-off of a tiny reusable guidance module on a fixed backbone. We compare against representative fixed-model inpainting baselines spanning several major adaptation directions: Unguided, Replacement \cite{song2020score}, MCG \cite{chung2022improving}, DPS \cite{chung2022diffusion}, and RePaint \cite{lugmayr2022repaint}. These baselines cover no conditioning, replacement-based heuristics, manifold-constrained guidance, posterior-gradient guidance, and iterative resampling respectively. In particular, RePaint is a strong and widely used benchmark in this line of work. For the latent EDM2 setting, we also include LatentPaint \cite{corneanu2024latentpaint}, the closest reusable-module competitor available in our latent-backbone setting. We therefore evaluate AID against both strong fixed-model methods and the most relevant latent-space amortized baseline. We report masked-region PSNR together with SSIM \cite{wang2004image} and LPIPS \cite{zhang2018unreasonable} on the completed image, using two AID operating points: the default setting \(K=18\) (AID-18) for best quality and the low-latency setting \(K=12\) (AID-12) for faster sampling. Importantly, the low-latency setting \(K=12\) is obtained by directly deploying the same guidance module trained for the default \(K=18\) setting, without retraining.

\subsection{Pixel-Space EDM Benchmarks: AFHQv2 and FFHQ}
We first evaluate on AFHQv2 and FFHQ, the two pixel-space EDM image benchmarks. The main free-form results are summarized in Table~\ref{tab:main_results}, the complete frontier is shown in Figure~\ref{fig:frontier}, and a visual comparison is presented in Figure \ref{fig:qualitative_edm}. On both datasets, AID-18 achieves the strongest performance on every free-form metric among all compared methods while using the same NFE as the standard 35-NFE baselines and only one tenth of RePaint's sampling budget. The improvement is consistent across fidelity-oriented metrics and perceptual quality, indicating that the learned guidance does not merely improve visible-region matching but also improves the overall completed image. The margins over Replacement, MCG, and DPS are larger still, suggesting that the gain cannot be explained by simple overwrite, projection, or posterior-guidance heuristics.

\begin{table}[t]
\centering
\caption{Free-form inpainting results on 1000 test images. Entries report mean \(\pm\) standard deviation over three seeds; LatentPaint is included only for ImageNet.}
\label{tab:main_results}
\begin{tabular}{llcccccc}
\toprule
Data & Method & NFE & Sec. & PSNR\,$\uparrow$ & SSIM\,$\uparrow$ & LPIPS\,$\downarrow$ \\
\midrule
\multicolumn{7}{l}{\textbf{AFHQv2}} \\
& Unguided       & 35  & 1.54  & 4.85 $\pm$ 0.06 & 0.4296 $\pm$ 0.0035 & 0.2716 $\pm$ 0.0025 \\
& Replace       & 35  & 1.55  & 7.15 $\pm$ 0.08 & 0.5164 $\pm$ 0.0052 & 0.1931 $\pm$ 0.0026 \\
& MCG            & 35  & 1.58  & 7.08 $\pm$ 0.02 & 0.5172 $\pm$ 0.0053 & 0.1915 $\pm$ 0.0015 \\
& DPS            & 35  & 2.41  & 10.00 $\pm$ 0.02 & 0.6336 $\pm$ 0.0048 & 0.1257 $\pm$ 0.0019 \\
& RePaint        & 350 & 15.50 & 10.97 $\pm$ 0.03 & 0.6816 $\pm$ 0.0031 & 0.0879 $\pm$ 0.0006 \\
& AID-12         & \textbf{23}  & \textbf{1.03}  & 11.09 $\pm$ 0.04 & 0.6815 $\pm$ 0.0035 & 0.0955 $\pm$ 0.0017 \\
& AID-18         & 35  & 1.57  & \textbf{13.01 $\pm$ 0.09} & \textbf{0.7336 $\pm$ 0.0034} & \textbf{0.0772 $\pm$ 0.0015} \\
\midrule
\multicolumn{7}{l}{\textbf{FFHQ}} \\
& Unguided       & 35  & 1.55  & 4.65 $\pm$ 0.00 & 0.4607 $\pm$ 0.0042 & 0.2152 $\pm$ 0.0030 \\
& Replace       & 35  & 1.56  & 6.93 $\pm$ 0.08 & 0.5585 $\pm$ 0.0050 & 0.1532 $\pm$ 0.0015 \\
& MCG            & 35  & 1.59  & 6.92 $\pm$ 0.04 & 0.5633 $\pm$ 0.0049 & 0.1507 $\pm$ 0.0022 \\
& DPS            & 35  & 2.40  & 10.41 $\pm$ 0.03 & 0.7028 $\pm$ 0.0027 & 0.0792 $\pm$ 0.0011 \\
& RePaint        & 350 & 15.50 & 10.88 $\pm$ 0.05 & 0.7276 $\pm$ 0.0031 & 0.0677 $\pm$ 0.0013 \\
& AID-12         & \textbf{23}  & \textbf{1.04}  & 11.96 $\pm$ 0.04 & 0.7306 $\pm$ 0.0025 & 0.0662 $\pm$ 0.0011 \\
& AID-18         & 35  & 1.58  & \textbf{12.87 $\pm$ 0.06} & \textbf{0.7598 $\pm$ 0.0027} & \textbf{0.0597 $\pm$ 0.0009} \\
\midrule
\multicolumn{7}{l}{\textbf{ImageNet}} \\
& Unguided       & 35  & 1.84  & 3.53 $\pm$ 0.07 & 0.4311 $\pm$ 0.0025 & 0.5084 $\pm$ 0.0025 \\
& Replace       & 35  & 1.94  & 5.34 $\pm$ 0.02 & 0.4669 $\pm$ 0.0025 & 0.4275 $\pm$ 0.0028 \\
& MCG            & 35  & 1.86  & 5.29 $\pm$ 0.08 & 0.4692 $\pm$ 0.0018 & 0.4242 $\pm$ 0.0018 \\
& DPS            & 35  & 2.50  & 7.26 $\pm$ 0.07 & 0.5010 $\pm$ 0.0024 & 0.4057 $\pm$ 0.0019 \\
& LatentPaint    & 35  & 2.01  & 6.47 $\pm$ 0.06 & 0.4874 $\pm$ 0.0025 & 0.3989 $\pm$ 0.0028 \\
& RePaint        & 350 & 19.06 & 7.56 $\pm$ 0.03 & 0.5204 $\pm$ 0.0019 & 0.3272 $\pm$ 0.0026 \\
& AID-12         & \textbf{23}  & \textbf{1.22}  & 9.67 $\pm$ 0.04 & 0.5491 $\pm$ 0.0092 & 0.3321 $\pm$ 0.0009 \\
& AID-18         & 35  & 1.92  & \textbf{9.79 $\pm$ 0.04} & \textbf{0.5498 $\pm$ 0.0088} & \textbf{0.3108 $\pm$ 0.0006} \\
\bottomrule
\end{tabular}
\end{table}

\begin{figure}[t]
\centering
\includegraphics[width=\linewidth]{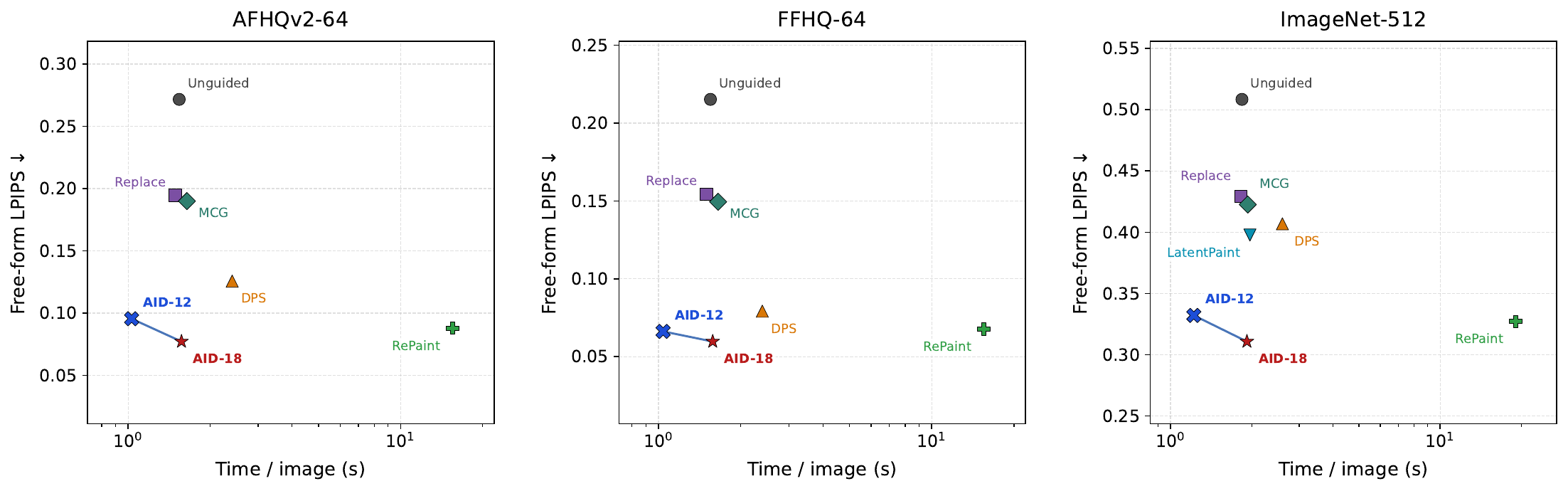}
\caption{Quality--speed frontier on free-form inpainting. Each panel corresponds to one dataset; ImageNet also includes LatentPaint. Lower LPIPS and lower wall-clock time are better.}
\label{fig:frontier}
\end{figure}

\begin{figure}[t]
\centering
\includegraphics[width=\linewidth]{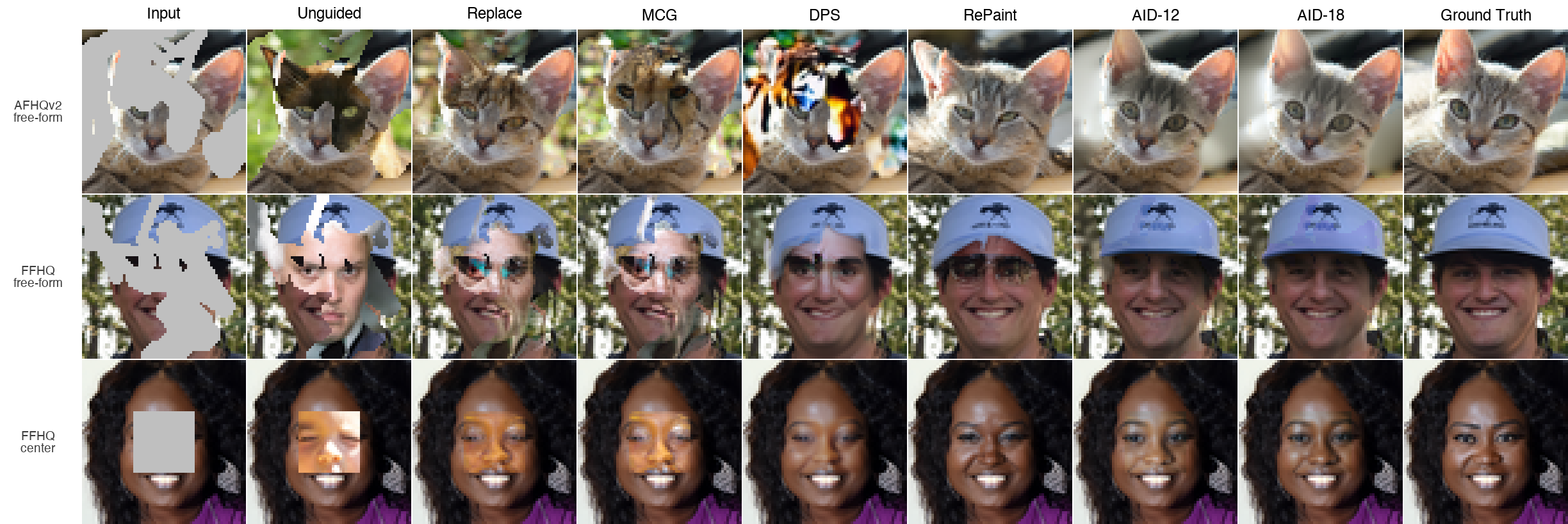}
\caption{Visual comparisons for pixel-space EDM. Additional results are provided in Appendix~\ref{sec:qualitative}.}
\label{fig:qualitative_edm}
\end{figure}

This advantage is not confined to the main free-form benchmark. In Appendix~\ref{sec:additional_results}, we show that AID-18 remains best on all metrics for both center and strip masks on AFHQv2 and FFHQ, implying  that the improvement is stable across unseen mask families rather than concentrated on a single setting. The low-latency operating point serves a different purpose: without any retraining, it substantially reduces inference time while still outperforming Unguided, Replacement, MCG, and DPS across these additional mask settings, and remains competitive with RePaint at only \(\frac{1}{15}\) of its sampling budget. The same learned guidance module can therefore support both a best-quality regime and a low-latency regime without reverting to per-instance optimization.

\subsection{Latent-Space EDM2 Benchmark: ImageNet}
We next evaluate on ImageNet using the pretrained EDM2 latent backbone. 
This setting is more demanding because the images have higher resolution \(512\times 512\), and sampling is performed in latent space whereas the final outputs are decoded by a VAE and evaluated in pixel space. It also provides the most relevant comparison to LatentPaint. As shown in Table~\ref{tab:main_results}, on the free-form benchmark, AID-18 improves over RePaint on all three metrics while using only about one tenth as many function evaluations. It also outperforms LatentPaint on every reported metric, despite using a guidance module with less than \(45\%\) of the trainable parameters of LatentPaint's propagation module. See also a visualization in Figure \ref{fig:qualitative_edm2}.

\begin{figure}[t]
\centering
\includegraphics[width=\linewidth]{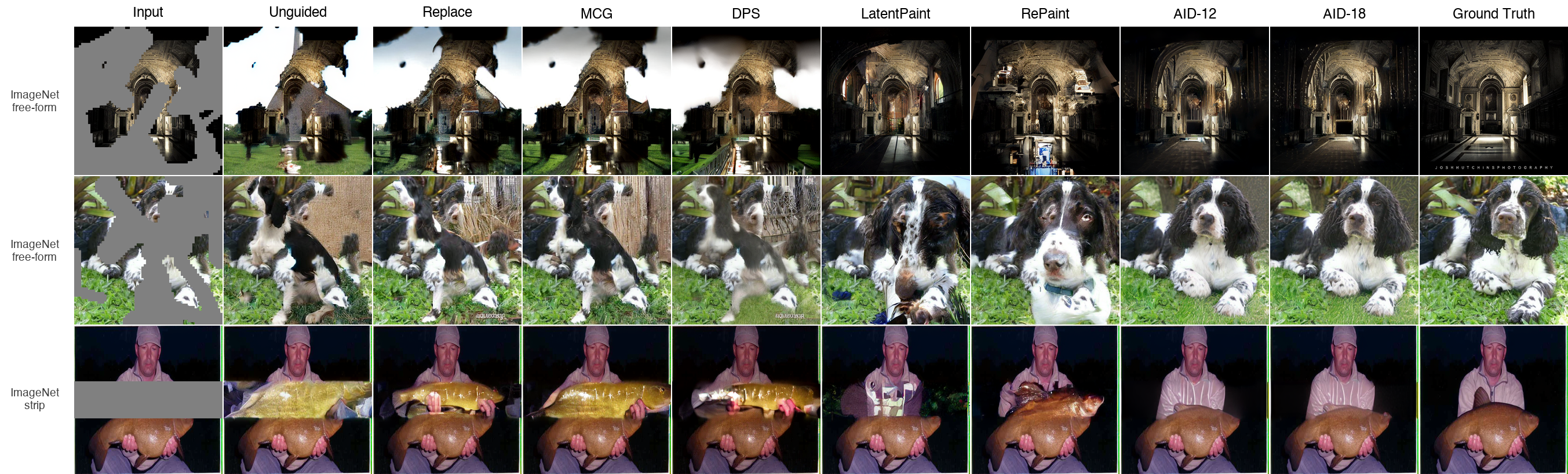}
\caption{Visual comparisons for latent-space EDM2. Additional results are provided in Appendix~\ref{sec:qualitative}.}
\label{fig:qualitative_edm2}
\end{figure}

The higher-resolution setting also sharpens the speed-quality story. The low-latency AID-12 operating point already surpasses all non-RePaint baselines and LatentPaint on every metric with significant margins, and it also improves over RePaint on PSNR and SSIM while giving up only a small amount on LPIPS. Across unseen center and strip masks shown in Appendix \ref{sec:additional_results}, the same pattern remains visible: the default point is the strongest overall configuration, while the low-latency point still improves substantially over the non-RePaint baselines and LatentPaint, and remains competitive with RePaint. The same amortized-guidance principle therefore remains effective when the pretrained model is replaced by the substantially more demanding latent EDM2 backbone at \(512\times 512\).

\subsection{Efficiency and Parameter Overhead}
The quality gains above are accompanied by a favorable deployment profile. The trainable module is tiny relative to the pretrained score model: \(409\)K parameters on EDM (\(0.66\%\) of the score network) and \(411\)K on EDM2 (\(0.33\%\)). On ImageNet this is also less than \(45\%\) the size of LatentPaint's learned propagation module, while giving better results on every reported metric and mask type. Within each dataset/backbone, the \(K=18\) operating point has essentially the same wall-clock cost as the 35-NFE baselines, but delivers markedly better quality. Compared with RePaint, it provides roughly a \(10\times\) speedup across all three datasets. The \(K=12\) operating point pushes further along the same frontier and reaches \(15\times\) speedup over RePaint while remaining competitive in quality. Taken together, these results support the main empirical claim of the paper: amortizing the guidance module moves up the entire quality--speed frontier of pretrained diffusion models, rather than merely recovering a different point on the same curve. A visual breakdown of the parameter overhead is deferred to Appendix \ref{sec:overhead}.

\section{Conclusion}
\label{sec:conclusion}

In this paper we introduce AID, a continuous-time actor--critic framework for image inpainting that keeps the diffusion backbone fixed and learns a small reusable guidance module offline. We formulate amortized inpainting as a deterministic guidance control problem and establish a policy-equivalence bridge to an auxiliary Gaussian formulation, preserving the deterministic guidance field used at deployment. Empirically, across AFHQv2, FFHQ, and ImageNet under the pixel EDM and latent EDM2 pipelines, AID improves the quality--speed trade-off over strong fixed-backbone and amortized inpainting baselines while adding less than one percent trainable overhead.

More broadly, this work adds to a recent line of research suggesting that control- and RL-based frameworks can provide a theoretically grounded  way to study and improve diffusion models. While this direction is still in its early innings,  our results indicate that continuous-time RL underpins not only rigorous analysis but also practical algorithm design in diffusion-based generation. We hope AID can motivate further work on controlled guidance, amortized adaptation, and broader uses of continuous-time RL for generative AI.

\bibliographystyle{plainnat}
\bibliography{ref}

\appendix
\setcounter{figure}{0}
\setcounter{table}{0}
\renewcommand{\thefigure}{A\arabic{figure}}
\renewcommand{\thetable}{A\arabic{table}}
\renewcommand{\theHfigure}{appendix.\arabic{figure}}
\renewcommand{\theHtable}{appendix.\arabic{table}}





\section{Proof of Lemma~\ref{lem:value_shift}}
\label{app:proof_value_shift}

Recall that, for a fixed observable input $\xi=(M,y)$, the deterministic value function $V$ solves
\begin{equation}
V_t(t,x;\xi)
+
\inf_{u\in\R^d}
\left\{
\nabla_x V(t,x;\xi)^\top \bigl(b(t,x)+u\bigr)
+
\frac{\beta}{2}\norm{u}^{2}
\right\}
=0,
\label{eq:appendix_hjb_det}
\end{equation}
with terminal condition
\begin{equation}
V(T,x;\xi)=\bar\Psi(x;\xi),
\label{eq:appendix_terminal_det}
\end{equation}
whereas the centered auxiliary value function $V^{(\lambda)}$ solves
\begin{equation}
V_t^{(\lambda)}(t,x;\xi)
+
\inf_{\mu\in\R^d}
\left\{
\nabla_x V^{(\lambda)}(t,x;\xi)^\top \bigl(b(t,x)+\mu\bigr)
+
\frac{\beta}{2}\norm{\mu}^{2}
+
\lambda
\right\}
=0,
\label{eq:appendix_hjb_aux}
\end{equation}
with terminal condition
\begin{equation}
V^{(\lambda)}(T,x;\xi)=\bar\Psi(x;\xi)-\lambda T.
\label{eq:appendix_terminal_aux}
\end{equation}

Our goal is to show that
\[
V^{(\lambda)}(t,x;\xi)=V(t,x;\xi)-\lambda t.
\]

\begin{proof}
Define
\[
W(t,x;\xi):=V(t,x;\xi)-\lambda t.
\]
Then we have
\begin{equation}
W_t(t,x;\xi)=V_t(t,x;\xi)-\lambda,
\qquad
\nabla_x W(t,x;\xi)=\nabla_x V(t,x;\xi).
\label{eq:appendix_W_derivatives}
\end{equation}
Substituting \eqref{eq:appendix_W_derivatives} into the left-hand side of the auxiliary HJB \eqref{eq:appendix_hjb_aux} gives
\[
\begin{aligned}
&W_t(t,x;\xi)
+
\inf_{\mu\in\R^d}
\left\{
\nabla_x W(t,x;\xi)^\top \bigl(b(t,x)+\mu\bigr)
+
\frac{\beta}{2}\norm{\mu}^{2}
+
\lambda
\right\}
\\
=\;&
V_t(t,x;\xi)-\lambda
+
\inf_{\mu\in\R^d}
\left\{
\nabla_x V(t,x;\xi)^\top \bigl(b(t,x)+\mu\bigr)
+
\frac{\beta}{2}\norm{\mu}^{2}
+
\lambda
\right\}=0.
\end{aligned}
\]
Moreover,
\[
W(T,x;\xi)
=
V(T,x;\xi)-\lambda T
=
\bar\Psi(x;\xi)-\lambda T.
\]
Hence $W$ solves \eqref{eq:appendix_hjb_aux} with terminal condition \eqref{eq:appendix_terminal_aux}. The conclusion now follows from 
the uniqueness of classical solutions to \eqref{eq:appendix_hjb_aux}.
\end{proof}

\section{Proof of Theorem~\ref{thm:policy_sufficiency}}
\label{app:proof_policy}

\begin{proof}
Fix an arbitrary $\pi^{(\lambda)}\in \Pi^{(\lambda)}(\xi)$, and denote its mean by
\[
\mu^\pi(t,x;\xi)
:=
\int_{\R^d} a\,\pi^{(\lambda)}(a\mid t,x;\xi)\,\mathrm{d}a.
\]
Since \(\pi^{(\lambda)}(\cdot\mid t,x;\xi)\) has mean
\(\mu^\pi(t,x;\xi)\) and covariance \(\frac{2\lambda}{\beta d}\Id\), we have
\begin{equation}
\int_{\R^d}\frac{\beta}{2}\norm{a}^{2}\,
\pi^{(\lambda)}(a\mid t,x;\xi)\,\mathrm{d}a
=
\frac{\beta}{2}\norm{\mu^\pi(t,x;\xi)}^2+\lambda.
\label{eq:appendix_running_cost_identity}
\end{equation}
Hence the centered auxiliary objective can be rewritten as
\begin{equation}
J^{\pi^{(\lambda)}}(t,x;\xi)
=
\bar\Psi\bigl(x^{\pi^{(\lambda)}}(T;t,x,\xi);\xi\bigr)
-
\lambda T
+
\int_t^T
\left(
\frac{\beta}{2}\norm{\mu^\pi(r,x^{\pi^{(\lambda)}}(r);\xi)}^2+\lambda
\right)\mathrm{d}r.
\label{eq:appendix_J_pi_rewrite}
\end{equation}

Applying the chain rule gives
\[
\frac{\mathrm{d}}{\mathrm{d}r}
V^{(\lambda)}\bigl(r,x^{\pi^{(\lambda)}}(r);\xi\bigr)
=
V_t^{(\lambda)}\bigl(r,x^{\pi^{(\lambda)}}(r);\xi\bigr)
+
\nabla_x V^{(\lambda)}\bigl(r,x^{\pi^{(\lambda)}}(r);\xi\bigr)^\top
\Bigl(
b+\mu^\pi
\Bigr),
\]
where, to lighten notation, we have suppressed the arguments \((r,x^{\pi^{(\lambda)}}(r),\xi)\) on the right-hand side.

Integrating from $t$ to $T$ and using the terminal condition of \(V^{(\lambda)}\), we obtain
\[
\bar\Psi\bigl(x^{\pi^{(\lambda)}}(T);\xi\bigr)-\lambda T-V^{(\lambda)}(t,x;\xi)
=
\int_t^T
\left[
V_t^{(\lambda)}
+
\nabla_x V^{(\lambda)}{}^\top (b+\mu^\pi)
\right]\mathrm{d}r.
\]
Adding \(\int_t^T \bigl(\frac{\beta}{2}\norm{\mu^\pi}^2+\lambda\bigr)\,\mathrm{d}r\) to both sides gives
\[
\begin{aligned}
&\bar\Psi\bigl(x^{\pi^{(\lambda)}}(T);\xi\bigr)-\lambda T-V^{(\lambda)}(t,x;\xi)
+
\int_t^T
\left(
\frac{\beta}{2}\norm{\mu^\pi}^2+\lambda
\right)\mathrm{d}r
\\
=\;&
\int_t^T
\left[
V_t^{(\lambda)}
+
\nabla_x V^{(\lambda)}{}^\top (b+\mu^\pi)
+
\frac{\beta}{2}\norm{\mu^\pi}^2
+
\lambda
\right]\mathrm{d}r.
\end{aligned}
\]
By \eqref{eq:appendix_J_pi_rewrite}, the left-hand side is exactly \(J^{\pi^{(\lambda)}}(t,x;\xi)-V^{(\lambda)}(t,x;\xi)\). Hence
\begin{equation}
J^{\pi^{(\lambda)}}(t,x;\xi)-V^{(\lambda)}(t,x;\xi)
=
\int_t^T
\left[
V_t^{(\lambda)}
+
\nabla_x V^{(\lambda)}{}^\top (b+\mu^\pi)
+
\frac{\beta}{2}\norm{\mu^\pi}^2
+
\lambda
\right]\mathrm{d}r.
\label{eq:appendix_verification_identity}
\end{equation}

Since \(V^{(\lambda)}\) solves the auxiliary HJB \eqref{eq:aux_hjb}, we have
\[
V_t^{(\lambda)}
+
\nabla_x V^{(\lambda)}{}^\top (b+\mu^\pi)
+
\frac{\beta}{2}\norm{\mu^\pi}^2
+
\lambda
\ge 0.
\]
Substituting this into \eqref{eq:appendix_verification_identity} yields
\[
J^{\pi^{(\lambda)}}(t,x;\xi)\ge V^{(\lambda)}(t,x;\xi).
\]
Since \(\pi^{(\lambda)}\in\Pi^{(\lambda)}(\xi)\) is arbitrary, this establishes that $V^{(\lambda)}$ is a lower bound of the value function of the auxiliary problem.

Now consider the deterministic Hamiltonian
\[
\mathcal H(u)
:=
\nabla_x V(t,x;\xi)^\top \bigl(b(t,x)+u\bigr)
+
\frac{\beta}{2}\norm{u}^{2},
\qquad
u\in\R^d.
\]
This is a {\it strictly} convex quadratic function of $u$, whose unique minimizer is
\begin{equation}
u^*(t,x;\xi)=-\frac{1}{\beta}\nabla_x V(t,x;\xi).
\label{eq:appendix_u_star}
\end{equation}
Meanwhile, the auxiliary Hamiltonian is
\[
\mathcal H^{(\lambda)}(\mu)
:=
\nabla_x V^{(\lambda)}(t,x;\xi)^\top \bigl(b(t,x)+\mu\bigr)
+
\frac{\beta}{2}\norm{\mu}^{2}
+
\lambda,
\]
whose unique minimizer is 
\[ -\frac{1}{\beta}\nabla_x V^{(\lambda)}(t,x;\xi)=-\frac{1}{\beta}\nabla_x V(t,x;\xi)=u^*(t,x;\xi), \]
where the first equality follows from Lemma~\ref{lem:value_shift}.

Now consider the specific Gaussian policy slice
\[
\pi^{(\lambda),*}(\cdot\mid t,x;\xi)
=
\mathcal N\!\left(
u^*(t,x;\xi),
\frac{2\lambda}{\beta d}\Id
\right).
\]
where \(u^*\) is given by \eqref{eq:appendix_u_star}. Because \(u^*(t,x;\xi)\) pointwise minimizes the auxiliary Hamiltonian, along the corresponding state trajectory the integrand on the right-hand side of \eqref{eq:appendix_verification_identity} is constantly zero. Thus
\[
J^{\pi^{(\lambda),*}}(t,x;\xi)=V^{(\lambda)}(t,x;\xi),
\]
implying that \(\pi^{(\lambda),*}\) is optimal for the auxiliary problem.

Finally, \(u^*(t,x;\xi)\) also minimizes the deterministic Hamiltonian; so the same argument gives that it is an optimal deterministic control for the original problem \eqref{eq:value}.
\end{proof}

\section{Proof of Theorem~\ref{thm:policy_necessity}}
\label{app:proof_learning_equivalence}

\begin{proof}
The proof of Theorem~\ref{thm:policy_sufficiency} establishes that, for any Gaussian policy \(\pi^{(\lambda)}\),
\[
J^{\pi^{(\lambda)}}(t,x;\xi)-V^{(\lambda)}(t,x;\xi)
=
\int_t^T
R^{\pi^{(\lambda)}}\bigl(r,x^{\pi^{(\lambda)}}(r;t,x,\xi);\xi\bigr)
\,\mathrm{d}r,
\]
where
\[
R^{\pi^{(\lambda)}}(r,z;\xi)
:=
V_t^{(\lambda)}(r,z;\xi)
+
\nabla_x V^{(\lambda)}(r,z;\xi)^\top
\bigl(b(r,z)+\mu^{\pi^{(\lambda)}}(r,z;\xi)\bigr)
+
\frac{\beta}{2}\norm{\mu^{\pi^{(\lambda)}}(r,z;\xi)}^2
+
\lambda .
\]
By the auxiliary HJB \eqref{eq:aux_hjb}, \(R^{\pi^{(\lambda)}}(r,z;\xi)\ge 0\).
Applying this identity to \(\pi^{(\lambda),*}\) and noting 
$
J^{\pi^{(\lambda),*}}(t,x;\xi)
=
V^{(\lambda)}(t,x;\xi), 
$
we have
\[
\int_t^T
R^{\pi^{(\lambda),*}}
\bigl(r,x^{\pi^{(\lambda),*}}(r;t,x,\xi);\xi\bigr)
\,\mathrm{d}r
=
0.
\]
Since the integrand above is nonnegative, the integral being zero implies that it vanishes almost everywhere along the trajectory. Because the initial pair \((t,x)\) is arbitrary and the residual is continuous, the residual vanishes pointwise. Hence
\(\mu^{\pi^{(\lambda),*}}(t,x;\xi)\) attains the pointwise minimum over \(\mu\) in \eqref{eq:aux_hjb}. However, the proof of Theorem~\ref{thm:policy_sufficiency} yields that the minimizer is unique and equals \(u^*(t,x;\xi)\) which is the optimal deterministic policy for the original problem.
\end{proof}

\section{Additional Quantitative Results}
\label{sec:additional_results}
Table~\ref{tab:center_results} reports the full center-mask results across all three datasets, and Table~\ref{tab:strip_results} reports the corresponding strip-mask results. These tables complement the free-form comparison in the main paper and support the robustness discussion in Section~\ref{sec:experiments}.

\paragraph{Runtime note.}
Wall-clock times are reported as empirical per-image sampling times under the same hardware and software environment for all methods. For each dataset, the reported time is averaged over the free-form, center, and strip mask settings. These numbers are intended as a secondary deployment-time indicator; the primary efficiency measure is NFE, since all methods within the same backbone use the same frozen score network and therefore have comparable per-evaluation cost. Timing values should mainly be interpreted for within-backbone comparisons, since the pixel-space EDM and latent-space EDM2 settings use different pretrained networks and representation spaces.

\begin{table}[H]
\centering
\caption{Average per-image sampling time across free-form, center, and strip masks.}
\label{tab:appendix_sampling_time}
\begin{tabular}{llccc}
\toprule
Method & NFE & AFHQv2 & FFHQ & ImageNet \\
\midrule
Unguided       & 35  & 1.54 & 1.55 & 1.84 \\
Replacement    & 35  & 1.55 & 1.56 & 1.94 \\
MCG            & 35  & 1.58 & 1.59 & 1.86 \\
DPS            & 35  & 2.41 & 2.40 & 2.50 \\
LatentPaint    & 35  & --   & --   & 2.01 \\
RePaint        & 350 & 15.50 & 15.50 & 19.06 \\
\midrule
AID (\(K=12\)) & \textbf{23}  & \textbf{1.03} & \textbf{1.04} & \textbf{1.22} \\
AID (\(K=18\)) & 35  & 1.57 & 1.58 & 1.92 \\
\bottomrule
\end{tabular}
\end{table}

\paragraph{Training note.}
Training the guidance module is lightweight relative to score-model pretraining. On the same Apple M2 Max / MPS environment, AID converges in roughly five hours on AFHQv2 and FFHQ and roughly eight hours on ImageNet, so even the largest reported configuration remains a single overnight run on laptop-class hardware. These numbers refer only to training the small guidance module on top of a fixed pretrained model. For context, the LaMa paper reports that \emph{Big LaMa}, their larger teaser model rather than the base model family, was trained on eight NVIDIA V100 GPUs for approximately 240 hours \cite[p.~6]{suvorov2022resolution}. LatentPaint is closer in spirit because it also avoids full model retraining, but its learned propagation module is latent-space specific and larger than our EDM2 guidance module; our ImageNet results therefore provide the most direct comparison to this middle-ground baseline.

All wall-clock times are measured in the same environment, so the relative comparisons are controlled by the same hardware and implementation stack. On Colab T4 GPU, absolute runtimes change, but the relative ordering is similar because the dominant difference comes from the number of function evaluations.

\begin{figure}[H]
\centering
\includegraphics[width=\linewidth]{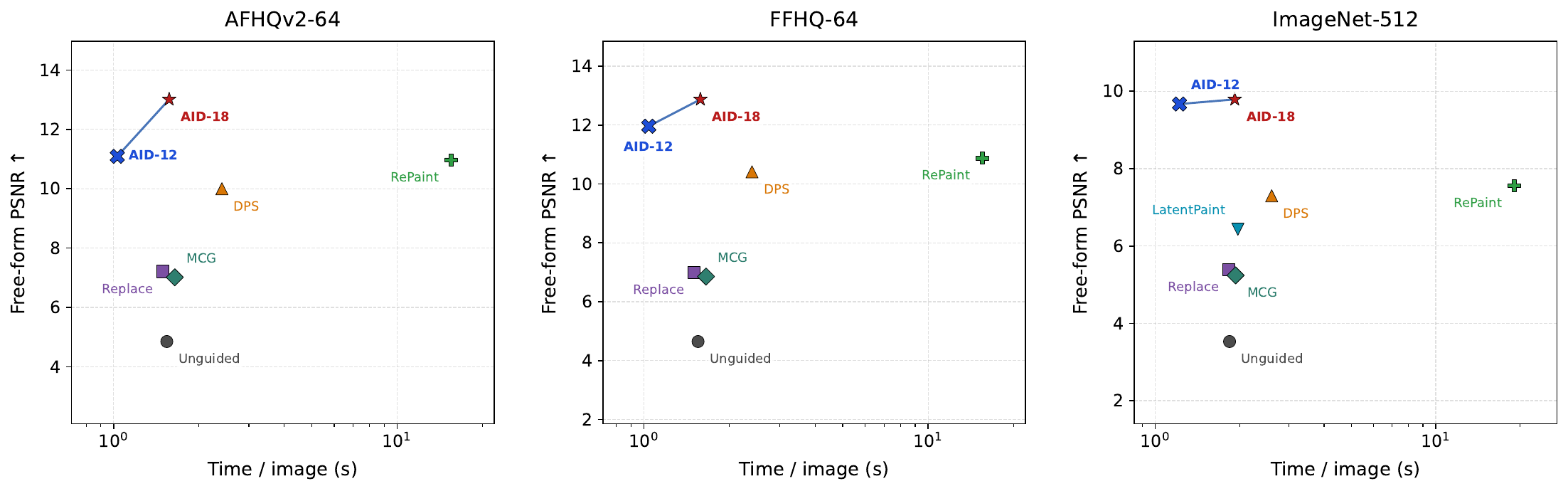}
\caption{Quality--speed frontier on the free-form benchmark measured by PSNR. Higher PSNR and lower wall-clock time are better.}
\label{fig:frontier_psnr_appendix}
\end{figure}

\begin{figure}[H]
\centering
\includegraphics[width=\linewidth]{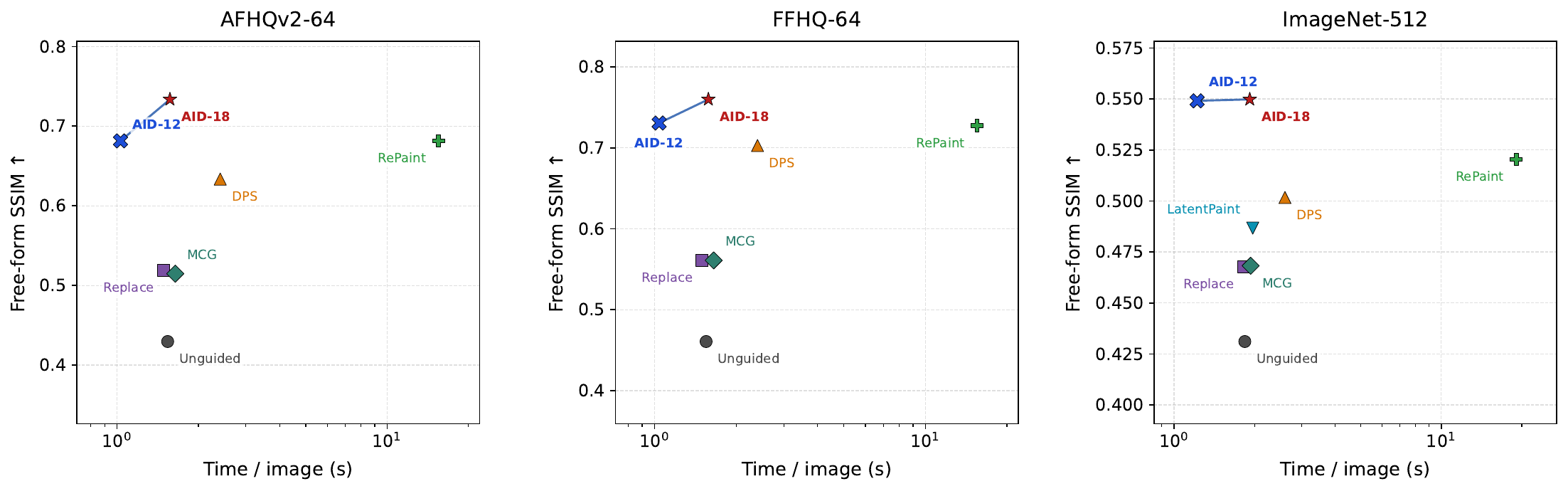}
\caption{Quality--speed frontier on the free-form benchmark measured by SSIM. Higher SSIM and lower wall-clock time are better.}
\label{fig:frontier_ssim_appendix}
\end{figure}

\begin{table}[H]
\centering
\caption{Center-mask inpainting results across the three datasets (100 test images per dataset). Entries are \textbf{mean $\pm$ std} over three random seeds. \textbf{Bold}: best.}
\label{tab:center_results}
\begin{tabular}{llccccc}
\toprule
Data & Method & NFE & Sec. & PSNR\,$\uparrow$ & SSIM\,$\uparrow$ & LPIPS\,$\downarrow$ \\
\midrule
\multicolumn{7}{l}{\textbf{AFHQv2}} \\
& Unguided       & 35  & 1.54  & 5.17 $\pm$ 0.25 & 0.8032 $\pm$ 0.0016 & 0.1741 $\pm$ 0.0082 \\
& Replace       & 35  & 1.55  & 8.42 $\pm$ 0.15 & 0.8559 $\pm$ 0.0008 & 0.0719 $\pm$ 0.0013 \\
& MCG            & 35  & 1.58  & 8.37 $\pm$ 0.39 & 0.8573 $\pm$ 0.0029 & 0.0698 $\pm$ 0.0031 \\
& DPS            & 35  & 2.41  & 10.11 $\pm$ 0.20 & 0.8839 $\pm$ 0.0016 & 0.0663 $\pm$ 0.0048 \\
& RePaint        & 350 & 15.50 & 13.70 $\pm$ 0.11 & 0.9293 $\pm$ 0.0026 & 0.0289 $\pm$ 0.0016 \\
& AID-12         & \textbf{23}  & \textbf{1.03}  & 12.08 $\pm$ 0.15 & 0.9153 $\pm$ 0.0004 & 0.0371 $\pm$ 0.0004 \\
& AID-18         & 35  & 1.57  & \textbf{14.54 $\pm$ 0.15} & \textbf{0.9352 $\pm$ 0.0016} & \textbf{0.0288 $\pm$ 0.0010} \\
\midrule
\multicolumn{7}{l}{\textbf{FFHQ}} \\
& Unguided       & 35  & 1.55  & 7.87 $\pm$ 0.08 & 0.8489 $\pm$ 0.0018 & 0.0986 $\pm$ 0.0019 \\
& Replace       & 35  & 1.56  & 11.49 $\pm$ 0.22 & 0.9020 $\pm$ 0.0027 & 0.0428 $\pm$ 0.0004 \\
& MCG            & 35  & 1.59  & 11.78 $\pm$ 0.10 & 0.9068 $\pm$ 0.0017 & 0.0405 $\pm$ 0.0009 \\
& DPS            & 35  & 2.40  & 14.83 $\pm$ 0.21 & 0.9384 $\pm$ 0.0021 & 0.0244 $\pm$ 0.0014 \\
& RePaint        & 350 & 15.50 & 15.64 $\pm$ 0.07 & 0.9496 $\pm$ 0.0013 & 0.0188 $\pm$ 0.0002 \\
& AID-12         & \textbf{23}  & \textbf{1.04}  & 15.54 $\pm$ 0.03 & 0.9465 $\pm$ 0.0006 & 0.0220 $\pm$ 0.0006 \\
& AID-18         & 35  & 1.58  & \textbf{16.93 $\pm$ 0.11} & \textbf{0.9577 $\pm$ 0.0005} & \textbf{0.0184 $\pm$ 0.0003} \\
\midrule
\multicolumn{7}{l}{\textbf{ImageNet}} \\
& Unguided       & 35  & 1.84  & 4.93 $\pm$ 0.15 & 0.6144 $\pm$ 0.0008 & 0.2158 $\pm$ 0.0018 \\
& Replace       & 35  & 1.94  & 6.39 $\pm$ 0.09 & 0.6235 $\pm$ 0.0001 & 0.1829 $\pm$ 0.0018 \\
& MCG            & 35  & 1.86  & 6.38 $\pm$ 0.03 & 0.6235 $\pm$ 0.0005 & 0.1823 $\pm$ 0.0013 \\
& DPS            & 35  & 2.50  & 7.34 $\pm$ 0.08 & 0.6169 $\pm$ 0.0004 & 0.2031 $\pm$ 0.0022 \\
& LatentPaint    & 35  & 2.01  & 7.24 $\pm$ 0.17 & 0.6275 $\pm$ 0.0011 & 0.1762 $\pm$ 0.0013 \\
& RePaint        & 350 & 19.06 & 7.88 $\pm$ 0.10 & 0.6335 $\pm$ 0.0005 & 0.1622 $\pm$ 0.0008 \\
& AID-12         & \textbf{23}  & \textbf{1.22}  & 9.62 $\pm$ 0.02 & 0.6371 $\pm$ 0.0006 & 0.1656 $\pm$ 0.0009 \\
& AID-18         & 35  & 1.92  & \textbf{9.86 $\pm$ 0.01} & \textbf{0.6375 $\pm$ 0.0009} & \textbf{0.1598 $\pm$ 0.0001} \\
\bottomrule
\end{tabular}
\end{table}

\begin{table}[H]
\centering
\caption{Strip-mask inpainting results across the three datasets (100 test images per dataset). Entries are \textbf{mean $\pm$ std} over three random seeds. \textbf{Bold}: best.}
\label{tab:strip_results}
\begin{tabular}{llccccc}
\toprule
Data & Method & NFE & Sec. & PSNR\,$\uparrow$ & SSIM\,$\uparrow$ & LPIPS\,$\downarrow$ \\
\midrule
\multicolumn{7}{l}{\textbf{AFHQv2}} \\
& Unguided       & 35  & 1.54  & 4.94 $\pm$ 0.28 & 0.6976 $\pm$ 0.0029 & 0.2030 $\pm$ 0.0078 \\
& Replace       & 35  & 1.55  & 7.60 $\pm$ 0.16 & 0.7529 $\pm$ 0.0038 & 0.1040 $\pm$ 0.0018 \\
& MCG            & 35  & 1.58  & 7.67 $\pm$ 0.38 & 0.7563 $\pm$ 0.0070 & 0.1036 $\pm$ 0.0072 \\
& DPS            & 35  & 2.41  & 10.17 $\pm$ 0.32 & 0.8071 $\pm$ 0.0045 & 0.0873 $\pm$ 0.0066 \\
& RePaint        & 350 & 15.50 & 12.37 $\pm$ 0.30 & 0.8654 $\pm$ 0.0039 & 0.0410 $\pm$ 0.0023 \\
& AID-12         & \textbf{23}  & \textbf{1.03}  & 11.42 $\pm$ 0.19 & 0.8452 $\pm$ 0.0050 & 0.0536 $\pm$ 0.0030 \\
& AID-18         & 35  & 1.57  & \textbf{13.91 $\pm$ 0.16} & \textbf{0.8837 $\pm$ 0.0014} & \textbf{0.0402 $\pm$ 0.0007} \\
\midrule
\multicolumn{7}{l}{\textbf{FFHQ}} \\
& Unguided       & 35  & 1.55  & 5.69 $\pm$ 0.18 & 0.7359 $\pm$ 0.0018 & 0.1414 $\pm$ 0.0023 \\
& Replace       & 35  & 1.56  & 8.73 $\pm$ 0.25 & 0.7972 $\pm$ 0.0036 & 0.0764 $\pm$ 0.0052 \\
& MCG            & 35  & 1.59  & 8.72 $\pm$ 0.19 & 0.8023 $\pm$ 0.0039 & 0.0751 $\pm$ 0.0023 \\
& DPS            & 35  & 2.40  & 12.30 $\pm$ 0.14 & 0.8632 $\pm$ 0.0051 & 0.0435 $\pm$ 0.0019 \\
& RePaint        & 350 & 15.50 & 12.91 $\pm$ 0.27 & 0.8823 $\pm$ 0.0027 & 0.0343 $\pm$ 0.0011 \\
& AID-12         & \textbf{23}  & \textbf{1.04}  & 13.68 $\pm$ 0.04 & 0.8857 $\pm$ 0.0023 & 0.0366 $\pm$ 0.0016 \\
& AID-18         & 35  & 1.58  & \textbf{14.83 $\pm$ 0.04} & \textbf{0.9033 $\pm$ 0.0008} & \textbf{0.0311 $\pm$ 0.0006} \\
\midrule
\multicolumn{7}{l}{\textbf{ImageNet}} \\
& Unguided       & 35  & 1.84  & 4.38 $\pm$ 0.14 & 0.5473 $\pm$ 0.0007 & 0.2981 $\pm$ 0.0045 \\
& Replace       & 35  & 1.94  & 5.73 $\pm$ 0.03 & 0.5612 $\pm$ 0.0007 & 0.2540 $\pm$ 0.0021 \\
& MCG            & 35  & 1.86  & 5.69 $\pm$ 0.09 & 0.5615 $\pm$ 0.0008 & 0.2520 $\pm$ 0.0022 \\
& DPS            & 35  & 2.50  & 6.89 $\pm$ 0.11 & 0.5629 $\pm$ 0.0014 & 0.2718 $\pm$ 0.0038 \\
& LatentPaint    & 35  & 2.01  & 6.78 $\pm$ 0.27 & 0.5681 $\pm$ 0.0012 & 0.2409 $\pm$ 0.0032 \\
& RePaint        & 350 & 19.06 & 7.20 $\pm$ 0.11 & 0.5758 $\pm$ 0.0005 & 0.2230 $\pm$ 0.0006 \\
& AID-12         & \textbf{23}  & \textbf{1.22}  & 9.00 $\pm$ 0.07 & 0.5857 $\pm$ 0.0005 & 0.2199 $\pm$ 0.0013 \\
& AID-18         & 35  & 1.92  & \textbf{9.32 $\pm$ 0.02} & \textbf{0.5863 $\pm$ 0.0003} & \textbf{0.2096 $\pm$ 0.0009} \\
\bottomrule
\end{tabular}
\end{table}

\section{Additional Qualitative Results}
\label{sec:qualitative}
\begin{figure}[H]
\centering
\includegraphics[width=\linewidth]{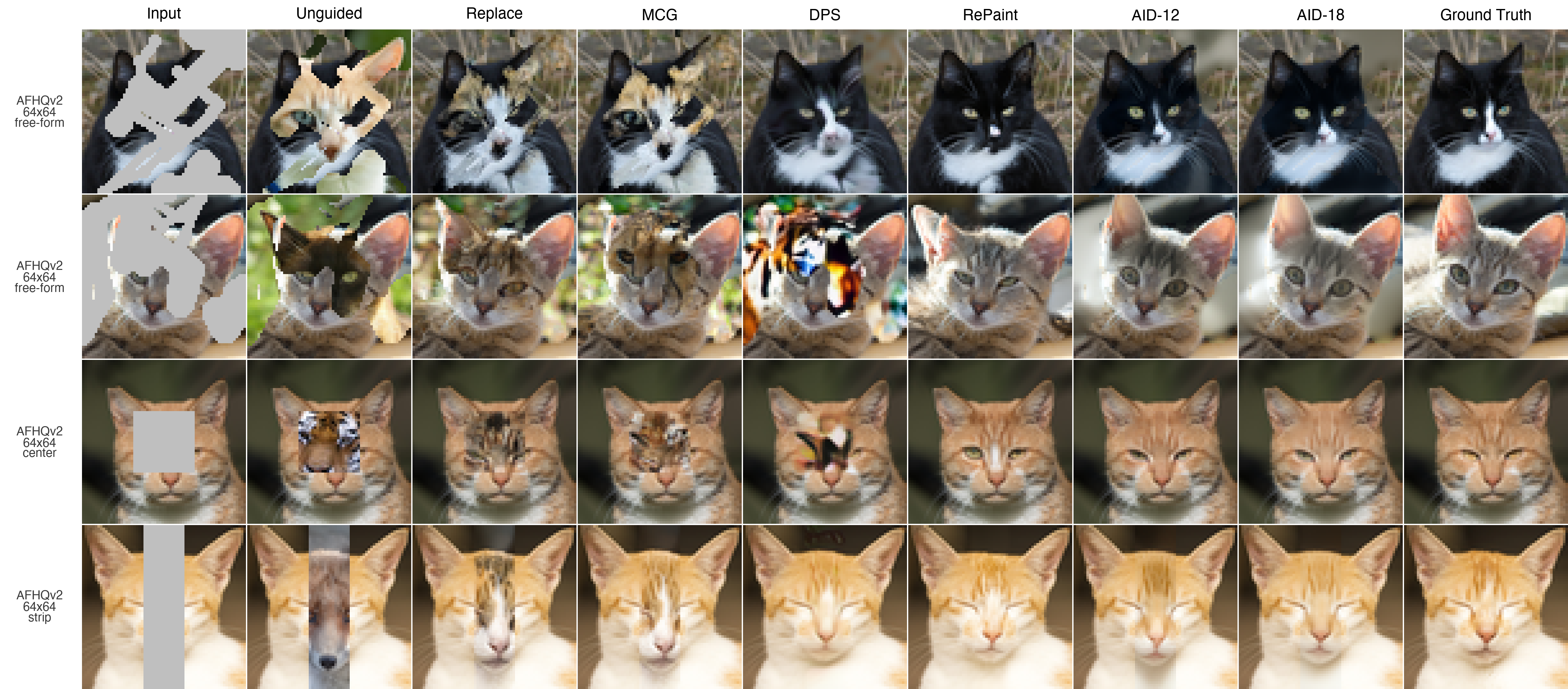}
\caption{Additional AFHQv2 qualitative comparisons across mask types. Columns follow the same method ordering as Figure~\ref{fig:qualitative_edm}.}
\label{fig:appendix_qual_afhq}
\end{figure}

\begin{figure}[H]
\centering
\includegraphics[width=\linewidth]{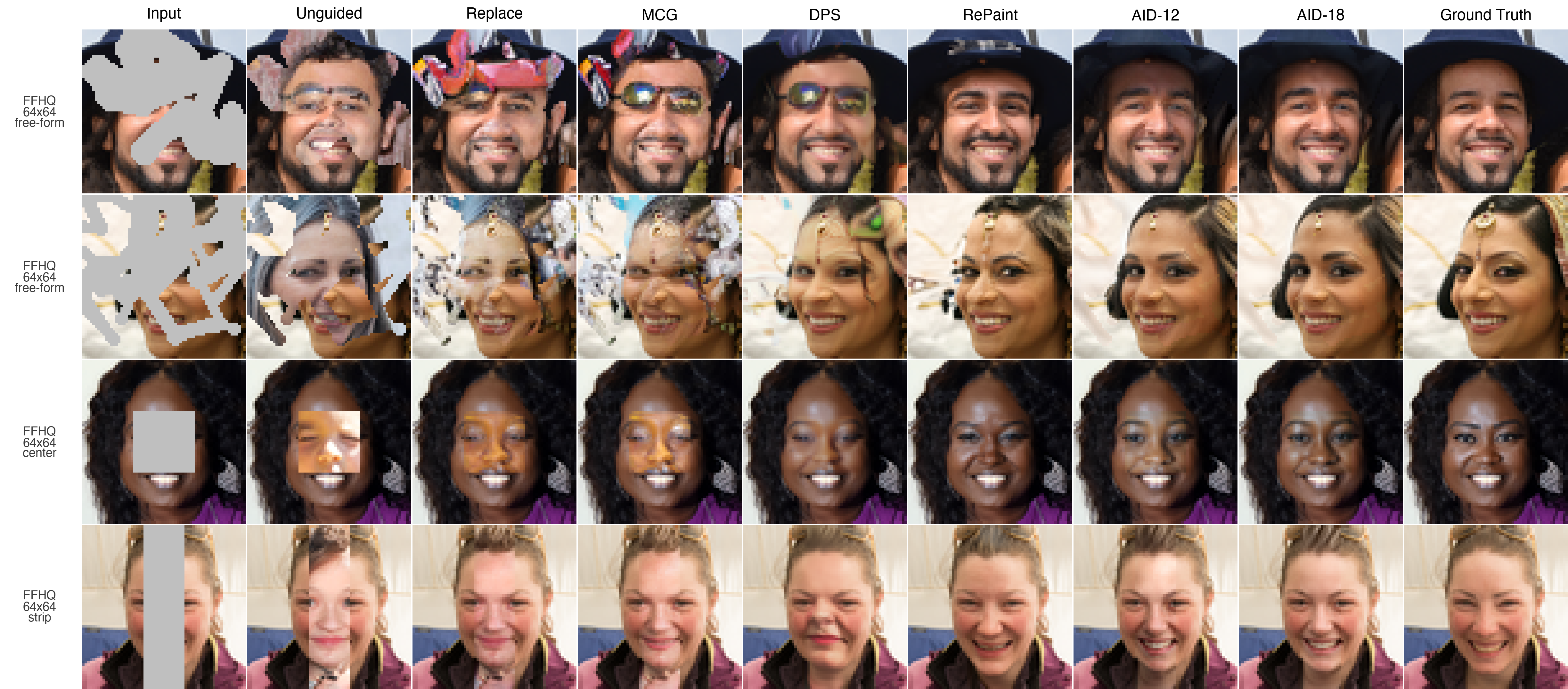}
\caption{Additional FFHQ qualitative comparisons across mask types. Columns follow the same method ordering as Figure~\ref{fig:qualitative_edm}.}
\label{fig:appendix_qual_ffhq}
\end{figure}

\begin{figure}[H]
\centering
\includegraphics[width=\linewidth]{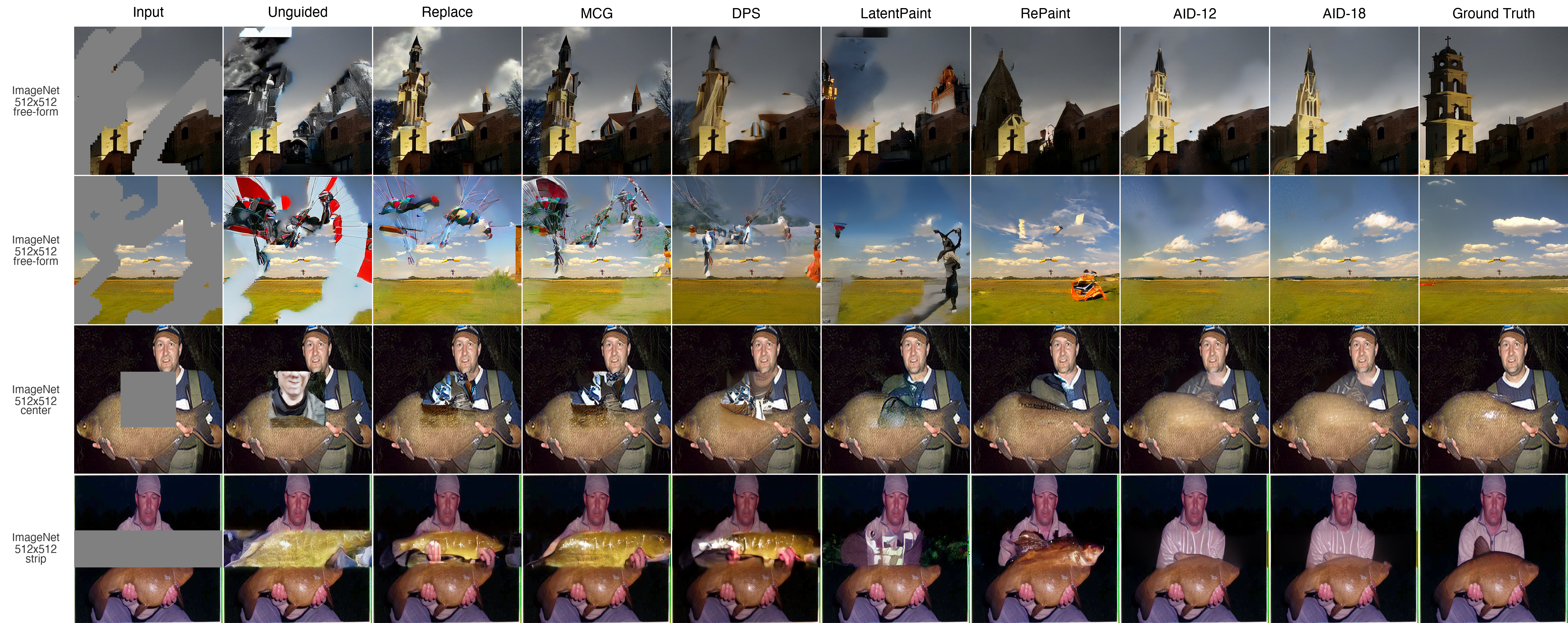}
\caption{Additional ImageNet qualitative comparisons across mask types. Columns follow the same method ordering as Figure~\ref{fig:qualitative_edm2}.}
\label{fig:appendix_qual_imagenet}
\end{figure}

\section{Parameter Overhead}
\label{sec:overhead}
\begin{figure}[H]
\centering
\includegraphics[width=0.9\linewidth]{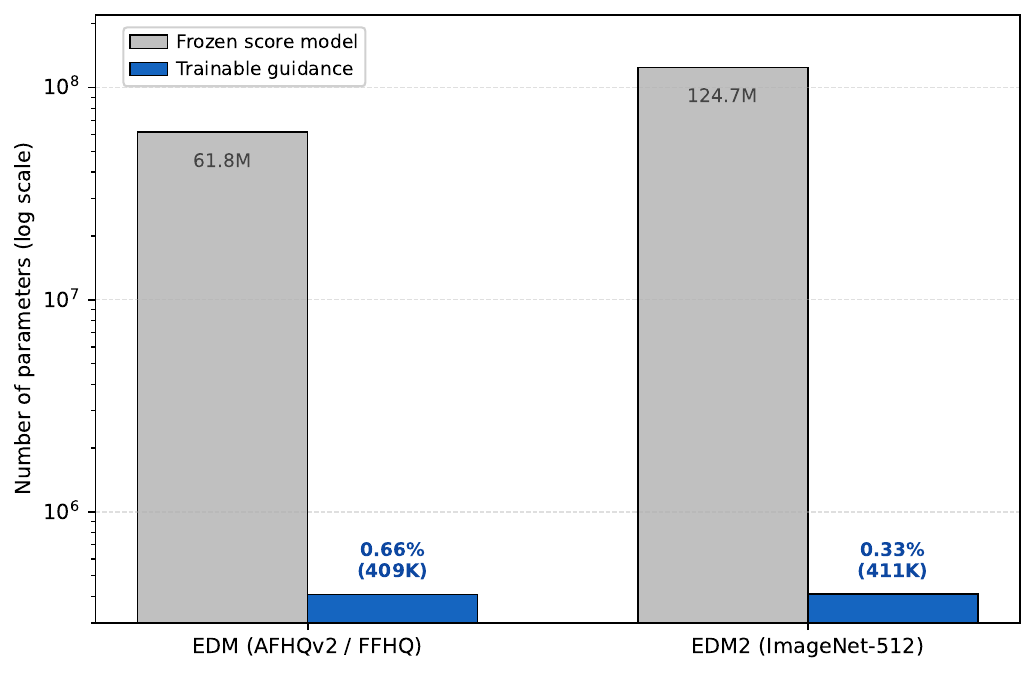}
\caption{Trainable parameter overhead relative to the frozen score model. AID remains below one percent of the score network in both backbone families, which is consistent with the deployment goal of learning a small reusable correction on top of a pretrained model rather than retraining the generative model itself.}
\label{fig:overhead_appendix}
\end{figure}

\section{Training Details}
\label{sec:training_details}

We summarize the training protocol used to produce the AID checkpoints in the main paper.

\paragraph{Network architecture.}
The actor \(\hat\mu^\phi\) and critic \(\hat V^\theta\) are two separate convolutional networks with no shared parameters. Both take \((t,x_t,\xi)\) and process \(\xi=(M,y)\) through a small conditioning stem (two residual blocks at width~\(32\)) whose output is concatenated with the state \(x_t\) and a sinusoidal \(\log\sigma(t)\) embedding; a four-block residual trunk at width~\(64\) is then applied. The actor projects the trunk output back to the state dimension; the critic spatially averages it and maps to a scalar. On AFHQv2 and FFHQ both networks operate in pixel space; on ImageNet they operate in the \(4\)-channel \(64\times 64\) EDM2 latent space. For ImageNet, masks and observations are represented on the EDM2 latent grid during sampling; completed latents are decoded to pixel space for PSNR, SSIM, and LPIPS evaluation using the corresponding pixel-space masks. The critic is used only at training time and is discarded at deployment. The actor has about \(409\)K trainable parameters on EDM and \(411\)K on EDM2, corresponding to \(0.66\%\) and \(0.33\%\) of the respective frozen score networks.

\paragraph{Frozen score backbones.}
We use the official pretrained pickles: EDM on AFHQv2 and FFHQ, and EDM2 on ImageNet \citep{karras2022elucidating,karras2024analyzing}. All score networks are kept frozen throughout training and deployment. For class-conditional ImageNet EDM2, we use the dataset labels associated with the sampled images.

\paragraph{Solver and time grid.}
The controlled reverse ODE~\eqref{eq:controlled} is discretized with Heun on the EDM \(\sigma\)-schedule (\(\rho=7\), \(\sigma_{\min}=0.002\), \(\sigma_{\max}=80\)). Training uses \(K=18\) steps; at deployment we consider \(K\in\{12,18\}\) with the same trained actor. Following the EDM convention, the final solver step is Euler, so the reported NFE is \(2K-1\): \(35\) for \(K=18\) and \(23\) for \(K=12\).

\paragraph{Task distribution.}
During training we sample \(\zeta=(M,y,x^\dagger)\sim\rho_{\mathrm{task}}\) online, drawing \(x^\dagger\) from the dataset's training split and \(M\) from a random free-form brush-mask distribution with total missing fraction uniform in \([0.2,0.6]\); the center and strip masks used for evaluation in Appendix~\ref{sec:additional_results} are never seen during training.

\paragraph{Optimizer.}
We use Adam for both networks with constant learning rate \(10^{-4}\), batch size~\(2\), gradient clipping at norm~\(1\), and no scheduler or EMA. Each run takes on the order of several hours on a single commodity GPU-equivalent device; detailed settings per dataset are provided in the supplementary configuration files.

\paragraph{Hyperparameter values (\(\alpha_{\mathrm{vis}},\alpha_{\mathrm{hole}},\beta,\lambda\)).}
Table~\ref{tab:training_hparams} lists the numerical values used for the released checkpoints. The control weight \(\beta\) is fixed to \(10^{-3}\) across all datasets. The exploration temperature is fixed to \(\lambda=10^{-3}\). We did not tune these hyperparameters carefully; the sensitivity study in Appendix~\ref{sec:ablation} confirms that moderate deviations do not change the ranking against the baselines of Table~\ref{tab:main_results}.

\begin{table}[H]
\centering
\caption{Training hyperparameters used for the released AID checkpoints.}
\label{tab:training_hparams}
\begin{tabular}{lcccc}
\toprule
Dataset / backbone & \(\alpha_{\mathrm{vis}}\) & \(\alpha_{\mathrm{hole}}\) & \(\beta\) & \(\lambda\) \\
\midrule
AFHQv2 / EDM        & 2 & 1 & \(10^{-3}\) & \(10^{-3}\) \\
FFHQ / EDM          & 2 & 1 & \(10^{-3}\) & \(10^{-3}\) \\
ImageNet / EDM2 & 1 & 1 & \(10^{-3}\) & \(10^{-3}\) \\
\bottomrule
\end{tabular}
\end{table}

\paragraph{Reproducibility.}
All training and evaluation runs use three random seeds whose specific values are fixed in the supplementary code. The complete training and evaluation configuration for each dataset is provided in the supplementary materials, and all results are fully reproducible.

\section{Ablation Study on AFHQv2}
\label{sec:ablation}

We ablate the four hyperparameters of the AID training objective: the control weight \(\beta\), the exploration temperature \(\lambda\), and the terminal weights \(\alpha_{\mathrm{vis}}\) and \(\alpha_{\mathrm{hole}}\). All ablations are conducted on AFHQv2 with the free-form mask (\(1000\) test images, same protocol as Table~\ref{tab:main_results}); the deployment setting is AID (\(K=18\)) throughout. In each table only the studied hyperparameter is varied, with perturbations at \(\times 2\) and \(\div 2\) of its default value; the remaining hyperparameters are held at the defaults of Table~\ref{tab:training_hparams}. For simplicity, we report only the resulting scalar metric values. The default row is consistent with the main AFHQv2 free-form result in Table~\ref{tab:main_results}. Throughout the four ablations, AID remains clearly ahead of the strongest baseline, RePaint (AFHQv2 free-form: \(10.97/0.6816/0.0879\)), on all three metrics.

\paragraph{Effect of the control weight \(\beta\).}
Table~\ref{tab:ablation_beta} varies \(\beta\) by a factor of two around the default \(\beta=10^{-3}\). A smaller \(\beta\) weakens the quadratic penalty on the guidance field and slightly softens all three metrics; a larger \(\beta\) regularizes more strongly and damps the guidance magnitude. Both directions produce only small shifts, indicating that the method is not highly sensitive to \(\beta\) within this range.

\begin{table}[H]
\centering
\caption{Ablation of the control weight \(\beta\) on AFHQv2 free-form (\(\alpha_{\mathrm{vis}}=2\), \(\alpha_{\mathrm{hole}}=1\), \(\lambda=10^{-3}\)).}
\label{tab:ablation_beta}
\begin{tabular}{lccc}
\toprule
\(\beta\) & PSNR\,$\uparrow$ & SSIM\,$\uparrow$ & LPIPS\,$\downarrow$ \\
\midrule
\(5\times 10^{-4}\)     & 12.94 & 0.7321 & 0.0780 \\
\(10^{-3}\) (default)   & \textbf{13.01} & \textbf{0.7336} & \textbf{0.0772} \\
\(2\times 10^{-3}\)     & 12.89 & 0.7306 & 0.0792 \\
\bottomrule
\end{tabular}
\end{table}

\paragraph{Effect of the exploration temperature \(\lambda\).}
Table~\ref{tab:ablation_lambda} varies \(\lambda\) with all other quantities fixed. Through the bridge of Section~\ref{sec:bridge}, \(\lambda\) enters the auxiliary policy variance \(2\lambda/(\beta d)\) and controls the amount of stochastic exploration used during training. A smaller \(\lambda\) reduces exploration, while a larger \(\lambda\) injects additional training-time noise; in both cases the actor remains stable, consistent with the theoretical prediction that the deterministic optimum \(u^*\) is recovered for any \(\lambda>0\).

\begin{table}[H]
\centering
\caption{Ablation of the exploration temperature \(\lambda\) on AFHQv2 free-form (\(\beta=10^{-3}\), \(\alpha_{\mathrm{vis}}=2\), \(\alpha_{\mathrm{hole}}=1\)).}
\label{tab:ablation_lambda}
\begin{tabular}{lccc}
\toprule
\(\lambda\) & PSNR\,$\uparrow$ & SSIM\,$\uparrow$ & LPIPS\,$\downarrow$ \\
\midrule
\(5\times 10^{-4}\)    & 12.97 & 0.7328 & 0.0776 \\
\(10^{-3}\) (default)  & \textbf{13.01} & \textbf{0.7336} & \textbf{0.0772} \\
\(2\times 10^{-3}\)    & 12.96 & 0.7325 & 0.0779 \\
\bottomrule
\end{tabular}
\end{table}

\paragraph{Effect of the visible-region weight \(\alpha_{\mathrm{vis}}\).}
Table~\ref{tab:ablation_alpha_vis} varies \(\alpha_{\mathrm{vis}}\) at \(\times 2\) and \(\div 2\) of the default \(\alpha_{\mathrm{vis}}=2\), with \(\alpha_{\mathrm{hole}}=1\) held fixed. Reducing \(\alpha_{\mathrm{vis}}\) weakens the visible-region supervision and softens PSNR, while doubling it gives a small improvement in PSNR with only minor changes in SSIM and LPIPS. The controller is therefore only mildly sensitive to \(\alpha_{\mathrm{vis}}\) within this range.

\begin{table}[H]
\centering
\caption{Ablation of the visible-region weight \(\alpha_{\mathrm{vis}}\) on AFHQv2 free-form (\(\alpha_{\mathrm{hole}}=1\), \(\beta=10^{-3}\), \(\lambda=10^{-3}\)).}
\label{tab:ablation_alpha_vis}
\begin{tabular}{lccc}
\toprule
\(\alpha_{\mathrm{vis}}\) & PSNR\,$\uparrow$ & SSIM\,$\uparrow$ & LPIPS\,$\downarrow$ \\
\midrule
\(1\)               & 12.82 & 0.7288 & 0.0800 \\
\(2\) (default)     & 13.01 & \textbf{0.7336} & \textbf{0.0772} \\
\(4\)               & \textbf{13.04} & 0.7330 & 0.0778 \\
\bottomrule
\end{tabular}
\end{table}

\paragraph{Effect of the hole-region weight \(\alpha_{\mathrm{hole}}\).}
Table~\ref{tab:ablation_alpha_hole} varies \(\alpha_{\mathrm{hole}}\) at \(\times 2\) and \(\div 2\) of the default \(\alpha_{\mathrm{hole}}=1\), with \(\alpha_{\mathrm{vis}}=2\) held fixed. Doubling \(\alpha_{\mathrm{hole}}\) slightly improves all three metrics, while halving it weakens missing-region supervision and produces the expected small regression.

\begin{table}[H]
\centering
\caption{Ablation of the hole-region weight \(\alpha_{\mathrm{hole}}\) on AFHQv2 free-form (\(\alpha_{\mathrm{vis}}=2\), \(\beta=10^{-3}\), \(\lambda=10^{-3}\)).}
\label{tab:ablation_alpha_hole}
\begin{tabular}{lccc}
\toprule
\(\alpha_{\mathrm{hole}}\) & PSNR\,$\uparrow$ & SSIM\,$\uparrow$ & LPIPS\,$\downarrow$ \\
\midrule
\(0.5\)             & 12.88 & 0.7307 & 0.0790 \\
\(1\) (default)     & 13.01 & 0.7336 & 0.0772 \\
\(2\)               & \textbf{13.06} & \textbf{0.7342} & \textbf{0.0770} \\
\bottomrule
\end{tabular}
\end{table}

\paragraph{Takeaway.}
Across all four hyperparameters, AID is only weakly sensitive within the \(\times 2\)/\(\div 2\) ranges considered, and in every perturbed configuration remains clearly ahead of RePaint, the strongest baseline, on AFHQv2 free-form. Several perturbations already match or slightly improve on the default row, which suggests that the method is robust to moderate changes in these training hyperparameters.

\end{document}